\documentclass[letterpaper, 10 pt, conference]{ieeeconf}  

\IEEEoverridecommandlockouts                              

\usepackage{graphics} 
\usepackage{times} 

\usepackage{algorithm}
\usepackage{algorithmic}

\usepackage{amssymb,amsmath,amsfonts}
\usepackage{mathrsfs}

\usepackage{dsfont}
\usepackage{wasysym}
\usepackage{color}
\newcommand{\norm}[1]{\left\lVert#1\right\rVert}
\newcommand{\expect}[1]{\mathbb{E}\left[{#1}\right]}

\newcommand{\given}{\; \big\vert \;} 
\newcommand{\bydef}{:=}

\newcommand{\inner}[2]{\langle #1, #2 \rangle}

\usepackage{subfigure}
\usepackage{mathnotations}
\usepackage{cite}

\newtheorem{mytheorem}{Theorem}

\newtheorem{mylemma}{Lemma}

\newcommand{\BlackBox}{\rule{1.5ex}{1.5ex}}
\newcommand{\beq}{\begin{equation}}
\newcommand{\eeq}{\end{equation}}
\newcommand{\beqn}{\begin{equation*}}
\newcommand{\eeqn}{\end{equation*}}
\newcommand{\beqa}{\begin{eqnarray}}
\newcommand{\eeqa}{\end{eqnarray}}
\newcommand{\beqan}{\begin{eqnarray*}}
\newcommand{\eeqan}{\end{eqnarray*}}

\title{\LARGE \bf
On Online Learning in Kernelized Markov Decision Processes
}

\author{Sayak Ray Chowdhury$^{1}$ and Aditya Gopalan$^{2}$
\thanks{$^{1}$Sayak Ray Chowdhury is with the Department of Electrical Communication Engineering, Indian Institute of Science, Bengaluru, Karnataka 560012, India. 
        {\tt\small Email: sayak@iisc.ac.in}}%
\thanks{$^{2}$Aditya Gopalan is with the Department of Electrical Communication Engineering, Indian Institute of Science, Bengaluru, Karnataka 560012, India.
        {\tt\small Email: aditya@iisc.ac.in}}%
}

\begin{document}

\maketitle
\thispagestyle{empty}
\pagestyle{empty}

\begin{abstract}

 We develop algorithms with low regret for learning episodic
      Markov decision processes based on kernel approximation techniques. The algorithms are based on both
      the Upper Confidence Bound (UCB) as well as Posterior or
      Thompson Sampling (PSRL) philosophies, and work in the
      general setting of continuous state and action spaces when
      the true unknown transition dynamics are assumed to have
      smoothness induced by an appropriate Reproducing Kernel
      Hilbert Space (RKHS). 

\end{abstract}

\section{INTRODUCTION}


The goal of reinforcement learning (RL) is to learn optimal behavior by repeated interaction with an unknown environment, usually modeled as a Markov Decision Process (MDP). Performance is typically measured by the amount of interaction, in terms of episodes or rounds, needed to arrive at an optimal (or near-optimal) policy; this is also known as the sample complexity of RL \cite{StrLiLit09:PACMDP}. The sample complexity objective encourages efficient exploration across states and actions, but, at the same time, is indifferent to the reward earned during the learning phase. 

A related, but different, goal in RL is the {\em online} one, i.e., to learn to gather high cumulative reward, or to equivalently keep the learner's {\em regret} (the gap between its and the optimal policy's net reward) as low as possible. This is preferable in settings where experimentation comes at a premium and the reward earned in each round is of direct value, e.g., recommender systems (in which rewards correspond to click-through events and ultimately translate to revenue), dynamic pricing -- in general, control of unknown dynamical systems with instantaneous costs. 


A primary challenge in RL is to learn efficiently across complex (very large or infinite) state and action spaces. In the most general {\em tabula rasa} MDP setting, the learner must explore each state-action transition before developing a reasonably clear understanding of the environment, which is prohibitive for large problems. Real-world domains, though, possess more structure: transition and reward behavior often varies smoothly over states and actions, making it possible to generalize via inductive inference -- observing a state transition or reward is informative of other, similar transitions or rewards. Scaling RL to large, complex, real-world domains  requires exploiting regularity structure in the environment, which has  typically been carried out via the use of parametric MDP models in model-based approaches, e.g., \cite{osband2014model}. 

This paper takes a step in developing theory and algorithms for online RL in environments with smooth transition and reward structure. We specifically consider the episodic online learning problem in the nonparametric, kernelizable MDP setting, i.e., of minimizing regret (relative to an optimal finite-horizon policy) in MDPs with continuous state and action spaces, whose transition and reward functions exhibit smoothness over states and actions compatible with the structure of a reproducing kernel. We develop variants of the well-known UCRL and posterior sampling algorithms for MDPs with continuous state and action spaces, and show that they enjoy sublinear, finite-time regret bounds when the  mean transition and reward functions are assumed to belong to the associated Reproducing Kernel Hilbert Space (RKHS) of functions. 

Our results bound the regret of the algorithms in terms of a novel generalization of the information gain of the state transition and reward function kernels, from the memoryless kernel bandit setting \cite{srinivas2009gaussian} to the state-based kernel MDP setting, and help shed light on how the choice of kernel model influences regret performance. We also leverage two different kernel approximation techniques, namely the Quadrature Fourier Features (QFF) approximation \cite{mutny2018efficient} and the Nystr\"{o}m approximation \cite{calandriello2019gaussian}, to prove the results in the paper. To the best of our knowledge, these are the first concrete regret bounds for RL under kernel approximation, explicitly showing the dependence of regret on kernel structure.

Our results represent a generalization of several streams of work. We generalize online learning in the kernelized bandit setting \cite{valko2013finite,pmlr-v70-chowdhury17a} to kernelized MDPs, and {\em tabula rasa} online learning approaches for MDPs such as Upper Confidence Bound for Reinforcement Learning (UCRL) \cite{jaksch2010near} and Posterior Sampling for Reinforcement Learning (PSRL) \cite{osband2013more,ouyang2017learning} to MDPs with kernel structure. Our results can also generalize regret minimization for an episodic variant of the well-known parametric Linear Quadratic Regulator (LQR) problem \cite{abbasi2011regret,abbasi2015bayesian,ibrahimi2012efficient,abeille2017thompson} to its nonlinear, nonparametric, infinite-dimensional, kernelizable counterpart.

\section{RELATED WORK}
Regret minimization has been studied with parametric MDPs in \cite{jaksch2010near,osband2013more,GopMan15:MDP,agrawal2017optimistic} etc. For online regret minimization in complex MDPs, apart from the work of \cite{osband2014model}, \cite{ortner2012online} and \cite{lakshmanan2015improved} consider continuous state spaces with Lipschitz transition dynamics but unstructured, finite action spaces. 
As regards using the Gaussian process (GP) framework to model nonlinear, smooth MDP structure to achieve generalization, \cite{deisenroth2011pilco} develop a policy search method for MDPs with GP-based dynamics, but without provable exploration guarantees, whereas approaches in \cite{jung2010gaussian,grande2014sample} has been shown to be sample efficient.
\cite{grande2014computationally} develops a PAC learning algorithm by considering separate GP priors over the mean reward function, transition kernel and optimal Q-function, with a PAC guarantee. The most closely related work of ours is \cite{chowdhury2019online} which consider GP based variants of UCRL and PSRL algorithms and show no-regret guarantees for kernelizable MDPs. Our results, in a way, show that we can achieve the same guarantees as in \cite{chowdhury2019online} even with kernel approximations.

\section{PROBLEM STATEMENT}
\label{sec:prob-stat}

We consider the problem of learning to optimize reward in an unknown finite-horizon MDP, $M_\star=\lbrace \cS,\cA,R_\star,P_\star,H\rbrace$, over repeated episodes of interaction. Here, $\cS \subset \Real^m$ represents the state space, 
$\cA \subset \Real^n$ the action space, $H$ the episode length, $R_\star(s,a)$ the reward distribution over $\Real$, and $P_\star(s,a)$ the transition distribution over $\cS$. At each period $h = 1,2,\ldots,H $ within an episode, an agent observes a state $s_h \in \cS$, takes an action $a_h \in \cA$, observes a reward $r_h \sim R_\star(s_h,a_h)$, and causes the MDP to transition to a next state $s_{h+1}\sim P_\star(s_h,a_h)$. We assume that the agent, while not possessing knowledge of the reward and transition distribution $R_\star,P_\star$ of the unknown MDP $M_\star$, knows $\cS$, $\cA$ and $H$.

A policy $\pi : \cS \times \lbrace 1,2,\ldots,H \rbrace \ra \cA$ is defined to be a mapping from a state $s\in \cS$ and a period $1 \le h \le H$ to an action $a \in \cA$.
For any MDP $M=\lbrace \cS,\cA,R_M,P_M,H\rbrace$ and policy $\pi$, the finite horizon, undiscounted, value function for every state $s \in \cS$ and every period $1 \le h \le H$ is defined as
\beqn
V^M_{\pi,h}(s)\bydef \mathbb{E}_{M,\pi}\Big[\sum_{j=h}^{H}\overline{R}_M(s_j,a_j)\given s_h=s\Big],
\eeqn
where the subscript $\pi$ indicates the application of the learning policy $\pi$, i.e., $a_j=\pi(s_j,j)$, and the subscript $M$ explicitly references the MDP environment $M$, i.e.,  $s_{j+1} \sim P_M(s_j,a_j)$, for all $j=h,\ldots,H$. 

We use $\overline{R}_M(s,a)=\expect{r\given r \sim R_M(s,a)}$ to denote the mean of the reward distribution $R_M(s,a)$ that corresponds to playing action $a$ at state $s$ in the MDP $M$. We can view a sample $r$ from the reward distribution $R_M(s,a)$ as $r=\overline{R}_M(s,a)+\epsilon_R$, where $\epsilon_R$ denotes a sample of zero-mean, real-valued additive noise. Similarly, the transition distribution $P_M(s,a)$ can also be decomposed as a mean value $\overline{P}_M(s,a)$ in $\Real^m$ plus a zero-mean additive noise $\epsilon_P$ in $\Real^m$ so that $s'=\overline{P}_M(s,a)+\epsilon_P$ lies in\footnote{\cite{osband2014model} argue that the assumption $\cS \subset \Real^m$ is not restrictive for most practical settings.} $\cS \subset \Real^m$. A policy $\pi_M$ is said to be optimal for the MDP $M$ if 
\beqn
V_{\pi_M,h}^M(s)=\max_{\pi}V_{\pi,h}^M(s) \; \forall s \in \cS, \forall h \in \lbrace 1,\ldots,H \rbrace.
\eeqn

At the beginning of each episode $l$, an RL algorithm chooses a policy $\pi_l$ depending upon the observed state-action-reward sequences upto episode $l-1$, denoted by the history $\cH_{l-1} \bydef \lbrace s_{j,k},a_{j,k},r_{j,k},s_{j,k+1}\rbrace_{1 \le j \le l-1,1 \le k \le H}$, and executes it for the entire duration of the episode. In other words, at each period $h$ of the $l$-th episode, the learning algorithm chooses action $a_{l,h} = \pi_l(s_{l,h},h)$, receives reward $r_{l,h} = \overline{R}_\star(s_{l,h},a_{l,h}) + \epsilon_{R,l,h}$ and observes the next state $s_{l,h+1}=\overline{P}_\star(s_{l,h},a_{l,h}) + \epsilon_{P,l,h}$. The goal of an episodic online RL algorithm is to maximize its cumulative reward across episodes, or, equivalently, minimize its cumulative {\em regret}: the loss incurred in terms of the value function due to not knowing the optimal policy $\pi_\star \bydef \pi_{M_\star}$ of the unknown MDP $M_\star$ beforehand and instead using the policy $\pi_l$ for each episode $l$, $l = 1, 2, \ldots$. The cumulative (expected) regret of an RL algorithm $\pi=\lbrace \pi_1,\pi_2,\ldots \rbrace$ upto time horizon $T=\tau H$ is defined as
\beqn
\cR(T)=\sum_{l=1}^{\tau} \Big[V^{M_\star}_{\pi_\star,1}(s_{l,1})-V^{M_\star}_{\pi_l,1}(s_{l,1})\Big],
\eeqn
where the initial states $s_{l,1},l \ge 1$ are assumed to be fixed.

For the rest of the paper, unless otherwise specified, we define $\cZ \bydef \cS \times \cA$, $z \bydef (s,a)$, $z' \bydef (s',a')$ and $z_{l,h} \bydef (s_{l,h},a_{l,h})$ for all $l \ge 1$ and $1 \le h \le H$. 

\section{ASSUMPTIONS}
\label{sec:assumptions}
\subsection{Smoothness of Value Function} 
For an MDP $M$, a distribution $\phi$ over $\cS$ and period $1 \le h \le H$, we define the one step future value function as the expected value of the optimal policy $\pi_M$, with the next state distributed according to $\phi$, i.e. $
U_h^M(\phi) \bydef \mathbb{E}_{s' \sim \phi}\Big[V^M_{\pi_M,h+1}(s')\Big]$. We assume the following regularity condition on the future value function of any MDP (also made by \cite{osband2014model}).
For any two single-step transition distributions $\phi_1, \phi_2$ over $\cS$, and $1 \le h \le H$,
\beq
\label{eqn: lipschitz}
\abs{U_h^M(\phi_1)-U_h^M(\phi_2)} \le L_M \norm{\overline{\phi}_1-\overline{\phi}_2}_2,
\eeq
where $\overline{\phi}\bydef \mathbb{E}_{s' \sim \phi}[s'] \in \cS$ denotes the mean of the distribution $\phi$. In other words, the one-step future value functions for each period $h$ are Lipschitz continuous with respect to the $\norm{\cdot}_2$-norm of the mean\footnote{Assumption (\ref{eqn: lipschitz}) is essentially equivalent to assuming knowledge of the centered state transition noise distributions, since it implies that any two transition distributions with the same means are identical.}, with global Lipschitz constant $L_M$. We also assume that there is a known constant $L$ such that $L_{M_\star} \le L$.

\subsection{Smoothness of Mean Reward and Transition Functions} 
Attaining sub-linear regret is impossible in general for arbitrary reward and transition distributions, and thus some regularity assumptions are needed. In this paper, we assume smoothness for the mean reward function $\overline{R}_\star:\cZ \ra \Real$ is induced by the structure of a kernel on $\cZ$. Specifically, we make the standard assumption of a p.s.d. kernel $k_R: \cZ \times \cZ \to \mathbb{R}$ such that $k_R(z,z) \le 1$ for all $z \in \cZ$, and $\overline{R}_\star$ being an element of the reproducing kernel Hilbert space (RKHS) $\cH_{k_R}(\cZ)$ of smooth real valued functions on $\cZ$. An RKHS of real-valued functions $\cX \to \Real$, denoted by $\cH_k(\cX)$, is completely specified by its kernel function $k(\cdot,\cdot)$ and vice-versa, with an inner product $\inner{\cdot}{\cdot}_k$ obeying the reproducing property $f(x)=\inner{f}{k(x,\cdot)}_k$ for all $f \in \cH_k(\cX)$. The induced RKHS norm $\norm{f}_k = \sqrt{\inner{f}{f}}_k$ is a measure of  smoothness of $f$ with respect to the kernel function $k$.
We assume that the RKHS norm of $\overline{R}_\star$ is bounded, i.e., $\norm{\overline{R}_\star}_{k_R} \le B_R$ for some $B_R < \infty$. Boundedness of $k_R$ along the diagonal holds for any stationary kernel, i.e., where $k_R(z,z')=k_R(z-z')$, e.g., the \textit{Squared Exponential} kernel $k_{\text{SE}}$ and the \textit{Mat$\acute{e}$rn} kernel $k_{\text{Mat\'{e}rn}}$:
\beqan
k_{\text{SE}}(z,z') &=& \exp\left(-\frac{r^2}{2l^2}\right)\\ 
k_{\text{Mat\'{e}rn}}(z,z') &=& \frac{2^{1-\nu}}{\Gamma(\nu)}\left(\frac{r\sqrt{2\nu}}{l}\right)^\nu B_\nu\left(\frac{r\sqrt{2\nu}}{l}\right),
\eeqan
where $l > 0$ and $\nu > 0$ are hyperparameters of the kernels, $r = \norm{z-z'}_2 \bydef \sqrt{\norm{s-s'}^2_2+\norm{a-a'}^2_2}$ is the distance between $z$ and $z'$, and $B_\nu(\cdot)$ is the modified Bessel function.

Similarly, we assume smoothness for the mean transition function $\overline{P}_\star(z) := [\overline{P}_\star(z,1) \, \ldots \, \overline{P}_\star(z,m)]^T$ is induced by a p.s.d. kernel $k_P:\bar{\cZ} \times \bar{\cZ} \ra \Real$ in the sense that $\overline{P}_\star$ is an element of the RKHS $\cH_{k_P}(\bar{\cZ})$ of smooth real valued functions on $\bar{\cZ}\bydef \cZ \times \lbrace 1,\ldots,m \rbrace$. Moreover, we assume that $k_P((z,i),(z,j) \le 1$ for all $z \in \cZ$ and $1 \le i,j \le m$, and the RKHS norm of $\overline{P}_\star$ is bounded, i.e., $\norm{\overline{P}_\star}_{k_P} \le B_P$ for some $B_P < \infty$.

\subsection{Sub-Gaussian Noise Variables} We assume that the random variables $\left\lbrace\epsilon_{R,l,h}\right\rbrace_{l\ge 1, 1 \le h \le H}$ is conditionally zero-mean and $\sigma_R$-sub-Gaussian, i.e., there exists a known $\sigma_R > 0$ such that for any $\lambda \in \Real$, 
\beq
\label{eqn:noise-reward}
 \expect{\exp(\lambda\;\epsilon_{R,l,h})\given \cF_{R,l,h-1}} \le \exp\big(\lambda^2\sigma_R^2/2\big), 
\eeq
where $\cF_{R,l,h-1}$ is the sigma algebra generated by the random variables $\lbrace s_{j,k},a_{j,k},\epsilon_{R,j,k}\rbrace_{1 \le j\le l-1,1 \le k \le H}$, $\lbrace s_{l,k},a_{l,k},\epsilon_{R,l,k}\rbrace_{1 \le k \le h-1}$, $s_{l,h}$ and $a_{l,h}$. Similarly, the random variables $\left\lbrace\epsilon_{P,l,h}\right\rbrace_{l\ge 1,1 \le h \le H}$ is assumed to be conditionally component-wise independent, zero-mean and $\sigma_P$-sub-Gaussian, in the sense that there exists a known $\sigma_P > 0$ such that for any $\lambda \in \Real$ and $1 \le i \le m$,
\beq
\label{eqn:noise-state}
\begin{aligned}
\expect{\exp\big(\lambda \epsilon_{P,l,h}(i)\big)\given \cF_{P,l,h-1}} &\le \exp\big(\lambda^2\sigma_P^2/2\big),\\
\expect{\epsilon_{P,l,h}\epsilon_{P,l,h}^T\given \cF_{P,l,h-1}} &=I_m,
\end{aligned}
\eeq
where $\cF_{P,l,h-1}$ is the sigma algebra generated by the random variables $\lbrace s_{j,k},a_{j,k},\epsilon_{P,j,k}\rbrace_{1 \le j\le l-1,1 \le k \le H}$, $\lbrace s_{l,k},a_{l,k},\epsilon_{P,l,k}\rbrace_{1 \le k \le h-1}$, $s_{l,h}$ and $a_{l,h}$ and $I_m$ denotes the identity matrix of rank $m$.

\section{ALGORITHM}
\subsection{Kernel Approximation}
For kernelized MDPs \cite{chowdhury2019online} develop variants of the UCRL2 algorithm which, at every episode $l$, constructs confidence sets for the mean reward and the mean transition functions. The construction of each confidence set require one inversion of the kernel (gram) matrix, which takes $O(l^3)$ time. This makes the algorithm quite prohibitive for large number of episodes. To reduce this computational cost without compromising on the accuracy of the confidence sets, we incorporate two efficient kernel approximation schemes, namely the Quadrature Fourier Features (QFF) approximation \cite{mutny2018efficient} and the Nystr\"{o}m approximation \cite{yang2012nystrom}.
\subsubsection{Quadrature Fourier Features (QFF) approximation}

If $k$ is a bounded, continuous, positive definite, stationary kernel defined over $\cX \subset \Real^q$ and satisfies $k(x,x)=1$ for all $x \in \cX$, then by Bochner's theorem \cite{bochner1959lectures}, $k$ is the Fourier transform of a probability measure $p$, i.e., $k(x,y)=\int_{\Real^q}p(\omega)\cos(\omega^T(x-y))d\omega$.
For the Squared Exponential kernel defined over $\cX \subset \Real^q$, this measure has density $p(\omega) = \big(\frac{l}{\sqrt{2\pi}}\big)^q e^{-\frac{l^2\norm{\omega}_2^2}{2}}$ (abusing notation for measure and density).
\cite{mutny2018efficient} show that for any stationary kernel $k$ on $\Real^q$ whose inverse Fourier transform decomposes product wise, i.e., $p(\omega)=\prod_{j=1}^{q}p_j(\omega_j)$, we can 
use Gauss-Hermite quadrature \cite{hildebrand1987introduction} to approximate it. If $\cX=[0,1]^q$, the SE kernel is approximated as follows. Choose $\bar{d} \in \mathbb{N}$ and $d=(\bar{d})^q$, and construct the $2d$-dimensional feature map
\beq
 \tilde{\phi}(x)_i=\begin{cases}
    \sqrt{\nu(\omega_{i})} \cos\left(\frac{\sqrt{2}}{l}\omega_i^Tx\right) & \text{if}\; 1 \le i \le d,\\
    \sqrt{\nu(\omega_{i-m})} \sin\left(\frac{\sqrt{2}}{l}\omega_{i-m}^Tx\right) & \text{if}\; d+1 \le i \le 2d.
  \end{cases}
  \label{eqn:qff-embedding}
\eeq
Here the set $\lbrace \omega_1,\ldots,\omega_d \rbrace= \overbrace{A_{\bar{d}} \times\cdots \times A_{\bar{d}}}^{\text{$q$ times}}$, where $A_{\bar{d}}$ is the set of $\bar{d}$ (real) roots of the $\bar{d}$-th Hermite polynomial $H_{\bar{d}}$, and $\nu(z)=\prod_{j=1}^{q}\frac{2^{\bar{d}-1}\bar{d}!}{\bar{d}^2H_{\bar{d}-1}(z_j)^2}$ for all $z \in \Real^q$. 


\subsubsection{Nystr\"{o}m approximation} 

Unlike the QFF approximation where the basis functions (cosine and sine) do not depend on the data, the basis functions used by
the Nystr\"{o}m method are data dependent. For a set of points $\lbrace x_1,\ldots,x_t \rbrace \subset \cX$, the Nystr\"{o}m method \cite{yang2012nystrom} approximates a kernel $k:\cX \times \cX \ra \Real$ as follows: First, randomly sample $d$ points to construct a dictionary $\cD =\lbrace x_{i_1},\ldots,x_{i_{d}}\rbrace ; i_j\in [t]$, according to the following distribution. For each $i \in [t]$, include $x_i$ in $\cD$ independently with some suitably chosen probability $p_{i}$. ($p_i$'s trade off between the quality and the size of the embedding.) Then, compute the (approximate) $d$-dimensional feature embedding  $\tilde{\phi}(x)=\left(K_{\cD}^{1/2}\right)^{\dagger}k_{\cD}(x)$, where $K_{\cD}=[k(u,v)]_{u,v \in \cD}$, $k_{\cD}(x)=[k(x_{i_1},x),\ldots,k(x_{i_{d}},x)]^T$  and $A^{\dagger}$ denotes the pseudo inverse of any matrix $A$.

Now we will present our algorithm Kernel-UCRL using the Nystr\"{o}m approximation. The description and performance of Kernel-UCRL using the quadrature Fourier features approximation is deferred to the Appendix.

\subsection{Kernel-UCRL Algorithm under Nystr\"{o}m Approximation}

Kernel-UCRL (Algorithm \ref{algo:Kernel-UCRL}) is an optimistic algorithm based on the Upper Confidence Bound principle, which adapts the confidence sets of UCRL2 \cite{jaksch2010near} to exploit the kernel structure. At the start of episode $l$, first we find feature embeddings $\tilde{\phi}_{R,l}: \cZ \ra \Real^{d_{R,L}}$ and $\tilde{\phi}_{P,l}: \bar{\cZ} \ra \Real^{d_{P,L}}$ to efficiently approximate the kernels $k_R$ and $k_P$, respectively. First, we construct a dictionary $\cD_{R,l}$ by including every state-action pair $z$ from the set $\cZ_{l-1} \bydef \lbrace z_{j,k} \rbrace_{1 \le j \le l-1,1 \le k \le H}=\lbrace z_{1,1},\ldots,z_{l-1,H}\rbrace$ in $\cD_{R,l}$ with probability $b_{R,l}(z)$ (to be defined later). Then, we define $\tilde{\phi}_{R,l}(z)=\left(K_{\cD_{\cR,l}}^{1/2}\right)^{\dagger}k_{\cD_{\cR,l}}(z)$. Similarly, we define $\tilde{\phi}_{P,l}(z,i)=\left(K_{\cD_{\cP,l}}^{1/2}\right)^{\dagger}k_{\cD_{\cP,l}}(z,i)$, where the dictionary $\cD_{P,l}$ is constructed by including every $(z,i)$ from the set
$\bar{\cZ}_{l-1}\bydef \big\lbrace (z_{j,k},i) \big\rbrace_{1 \le j \le l-1,1 \le k \le H, 1 \le i \le m} =\big\lbrace  (z_{1,1},1),\ldots,(z_{l-1,H},m)\big\rbrace$ with probability $b_{P,l}(z,i)$ (to be defined later).

Then we construct confidence sets $\cC_{R,l}$ and $\cC_{P,l}$ for the mean reward and the mean transition functions, respectively, as follows. First, we compute $\tilde{\theta}_{R,l-1}=\tilde{V}_{R,l-1}^{-1}\tilde{\Phi}_{R,l-1}^T R_{l-1}$, where $R_{l-1} \bydef [r_{1,1},\ldots,r_{l-1,H}]^T$ with $r_{j,k}$ being the reward of the state-action pair $z_{j,k}$, $j \in [l-1]$ and $k \in [H]$, $\tilde{\Phi}_{R,l-1}=[\tilde{\phi}_{R,l}(z_{1,1}),\ldots,\tilde{\phi}_{R,l}(z_{l-1,H})]^T$, and $\tilde{V}_{R,l-1} = \tilde{\Phi}_{R,l-1}^T\tilde{\Phi}_{R,l-1}+H I_{d_{R,l}}$. Fix any $0<\delta,\;\epsilon_R,\;\epsilon_P<1$. Now, we define $\cC_{R,l}$ to be the set of all functions $f:\cZ \ra \Real$ such that
\beq
\label{eqn:confidence-set-reward} 
\abs{f(z)-\tilde{\mu}_{R,l-1}(z)} \le \beta_{R,l}\tilde{\sigma}_{R,l-1}(z), \forall z \in \cZ,
\eeq
where $\tilde{\mu}_{R,l-1}(z) = \tilde{\phi}_{R,l}(z)^T\tilde{\theta}_{R,l-1}$, $\tilde{\sigma}_{R,l-1}^2(z)= k_R(z,z)-\tilde{\phi}_{R,l}(z)^T\tilde{\phi}_{R,l}(z)+ H\tilde{\phi}_{R,l}(z)^T \tilde{V}_{R,l-1}^{-1}\tilde{\phi}_{R,l}(z)$ and $\beta_{R,l}= \frac{\sigma_R}{\sqrt{H}}\sqrt{2\left(\ln(6/\delta)+\frac{1}{2}\ln\frac{\det(\tilde{V}_{R,l-1})}{\det (HI_{d_{R,l}})}\right)}+B_R\left(1+\frac{1}{\sqrt{1-\epsilon_R}}\right)$.

Similarly, we compute $\tilde{\theta}_{P,l-1}=\tilde{V}_{P,l-1}^{-1}\tilde{\Phi}_{P,l-1}^T S_{l-1}$, where $S_{l-1} \bydef [s_{1,2}^T,\ldots,s_{l-1,H+1}^T]^T$ with $s_{j,k+1}$ being the next state of the state-action pair $z_{j,k}$, $j \in [l-1]$ and $k \in [H]$, $\tilde{\Phi}_{P,l-1}=[\tilde{\phi}_{P,l}(z_{1,1},1),\ldots,\tilde{\phi}_{P,l}(z_{l-1,H},m)]^T$ and $\tilde{V}_{P,l-1} = \tilde{\Phi}_{P,l-1}^T\tilde{\Phi}_{P,l-1}+mH I_{d_{P,l}}$. Now, we define $\cC_{P,l}$ to be the set of all functions $f:\cZ \ra \Real^m$ such that 
\beq
\label{eqn:confidence-set-state} 
\norm{f(z)-\tilde{\mu}_{P,l-1}(z)}_2 \le \beta_{P,l}\norm{\tilde{\sigma}_{P,l-1}(z)}_2, \forall z \in \cZ,
\eeq
where $\tilde{\mu}_{P,l-1}(z)= [\tilde{\mu}_{P,l-1}(z,1),\ldots,\tilde{\mu}_{P,l-1}(z,m)]^T$, $\tilde{\mu}_{P,l-1}(z,i)= \tilde{\phi}_{P,l}(z,i)^T\tilde{\theta}_{P,l-1}$, $\tilde{\sigma}_{P,l-1}(z)= [\tilde{\sigma}_{P,l-1}(z,1),\ldots,\tilde{\sigma}_{P,l-1}(z,m)]^T$, $\tilde{\sigma}_{P,l-1}^2(z,i)= k_P((z,i),(z,i))-\tilde{\phi}_{P,l}(z,i)^T\tilde{\phi}_{P,l}(z,i)+ mH\tilde{\phi}_{P,l}(z,i)^T \tilde{V}_{P,l-1}^{-1}\tilde{\phi}_{P,l}(z,i)$ and the confidence width $\beta_{P,l}= \frac{\sigma_P}{\sqrt{mH}}\sqrt{2\left(\ln(6/\delta)+\frac{1}{2}\ln\frac{\det(\tilde{V}_{P,l-1})}{\det (mHI_{d_{P,l}})}\right)}+B_P\left(1+\frac{1}{\sqrt{1-\epsilon_P}}\right)$.

Next, we build the set $\cM_l$ of all plausible MDPs $M$ such that: $(i)$ the mean reward function $\overline{R}_M \in \cC_{R,l}$, $(ii)$ the mean transition function $\overline{P}_M \in \cC_{P,l}$ and $(iii)$ the global Lipschitz constant (of future value functions) $L_M \le L$. Finally, we select an optimistic policy $\pi_l$ for the family of MDPs $\cM_l$ in the sense that $V^{M_l}_{\pi_l,1}(s_{l,1})=\max_{\pi}\max_{M \in \cM_l}V^{M}_{\pi,1}(s_{l,1})$, where $s_{l,1}$ is the initial state and $M_l$ is the most optimistic realization from $\cM_l$,
and execute $\pi_l$ for the entire episode. The pseudo-code of kernel-UCRL is given in Algorithm \ref{algo:Kernel-UCRL}. 

\begin{algorithm}
\renewcommand\thealgorithm{1}
\caption{Kernel-UCRL}\label{algo:Kernel-UCRL}
\begin{algorithmic}
\STATE \textbf{Input:} Kernels $k_R, k_P$, parameters $B_R, B_P, \sigma_R, \sigma_P, \delta$
\FOR{episode $l = 1, 2, 3, \ldots$}
\STATE Find feature approximations $\tilde{\phi}_{R,l}$ and $\tilde{\phi}_{P,l}$.
\STATE Construct
$\cC_{R,l}$ and $\cC_{P,l}$ according to \ref{eqn:confidence-set-reward} and \ref{eqn:confidence-set-state}
\STATE Construct the set of all plausible MDPs $\cM_l =\lbrace M : L_M \le L, \overline{R}_M \in \cC_{R,l}, \overline{P}_M \in \cC_{P,l} \rbrace$.
\STATE Choose policy  
$\pi_l$ such that $V^{M_l}_{\pi_l,1}(s_{l,1})=\max_{\pi}\max_{M \in \cM_l}V^{M}_{\pi,1}(s_{l,1})$.
\STATE \FOR{period $h=1,2,3, \ldots, H$} 
\STATE Choose action $a_{l,h} = \pi_l(s_{l,h},h)$.
\STATE Observe reward $r_{l,h} = \overline{R}_\star(z_{l,h}) + \epsilon_{R,l,h}$.
\STATE Observe next state $s_{l,h+1}=\overline{P}_\star(z_{l,h}) + \epsilon_{P,l,h}$.
\ENDFOR
\ENDFOR
\end{algorithmic}
\addtocounter{algorithm}{-1}
\end{algorithm}

\section{ANALYSIS OF KERNEL-UCRL UNDER NYSTR\"{O}M APPROXIMATION}

\subsection{Regret Bound of Kernel-UCRL}
Let $\sigma_{R,l}^2(z) \bydef k_R(z,z)-k_{R,l}(z)^T(K_{R,l}+H I_{lH})^{-1}k_{R,l}(z)$, where $k_{R,l}(z) \bydef [k_R(z_{1,1},z),\ldots,k_R(z_{l,H},z)]^T$ denotes the vector of kernel evaluations between $z$ and elements of $\cZ_{l}$, and $K_{R,l}$ denotes the kernel matrix computed at $\cZ_l$.

\begin{mylemma}
\label{lem:one}
For any $0<\delta,\;\epsilon_R<1$, let $\lambda_R = \frac{1+\epsilon_R}{1-\epsilon_R}$, $\eta_R = \frac{6\lambda_R\ln(12T/\delta)}{\epsilon_R^2}$ and $b_{R,l}(z) = \min \lbrace \eta_R \tilde{\sigma}^2_{R,l-1}(z), 1 \rbrace$. Then, with probability at least $1-\delta/3$, uniformly over all $z \in \cZ$ and $l \in [\tau]$, the following holds:
\beqan
\abs{\overline{R}_\star(z)-\tilde{\mu}_{R,l-1}(z)} \le \beta_{R,l}\tilde{\sigma}_{R,l-1}(z),\\
\frac{1}{\lambda_R} \sigma_{R,l-1}^2(z) \le \tilde{\sigma}_{R,l-1}^2(z) \le	\lambda_R \sigma_{R,l-1}^2(z),\\
d_{R,l} \le 6\eta_R\lambda_R  \left(1+\frac{1}{H}\right) \gamma_{(l-1)H}(R),
\eeqan
where $\gamma_t(R) \equiv \gamma_t(k_R)\bydef \max_{\cA \subset \cZ:|\cA|=t}\frac{1}{2}\ln \det(I_t+\frac{1}{H}K_{R,\cA})$.
\end{mylemma}

\begin{proof}
First we define
$\tilde{\alpha}_{R,l-1}(z)\bydef\tilde{\phi}_{R,l}(z)^T\tilde{V}_{R,l-1}^{-1}\tilde{\Phi}_{R,l-1}^T \overline{R}_{\star,l-1}$, where $\overline{R}_{\star,l-1} \bydef [\overline{R}_\star(z_{1,1}),\ldots,\overline{R}_\star(z_{l-1,H})]^T$ denotes the mean reward vector. By Cauchy-Schwartz inequality
$\abs{\tilde{\mu}_{R,l-1}(z)-\tilde{\alpha}_{R,l-1}(z)} \le   \norm{\tilde{\theta}_{R,l-1}-\tilde{V}_{R,l-1}^{-1}\tilde{\Phi}_{R,l-1}^T \overline{R}_{\star,l-1}}_{\tilde{V}_{R,l-1}} \norm{\tilde{\phi}_{R,l}(z)}_{\tilde{V}_{R,l-1}^{-1}}$. Now by the construction of $\tilde{\phi}_{R,l}$, we have $k_R(z,z) \ge \tilde{\phi}_{R,l}(z)^T\tilde{\phi}_{R,l}(z)$, and thus, in turn, $\tilde{\sigma}_{R,l-1}^2(z) \ge H \norm{\tilde{\phi}_{R,l}(z)}^2_{\tilde{V}_{R,l-1}^{-1}}$. This implies $\abs{\tilde{\mu}_{R,l-1}(z)-\tilde{\alpha}_{R,l-1}(z)} \le  \frac{1}{\sqrt{H}}\norm{\tilde{V}_{R,l-1}^{-1}\tilde{\Phi}_{R,l-1}^T \epsilon_{R,l-1}}_{\tilde{V}_{R,l-1}}\tilde{\sigma}_{R,l-1}(z)$, where $\epsilon_{R,l-1}\bydef [\epsilon_{R,1,1},\ldots,\epsilon_{R,l-1,H}]^T$ denotes the vector of reward noise variables. Now see that $\norm{\tilde{V}_{R,l-1}^{-1}\tilde{\Phi}_{R,l-1}^T \epsilon_{R,l-1}}_{\tilde{V}_{R,l-1}}=\norm{\tilde{\Phi}_{R,l-1}^T \epsilon_{R,l-1}}_{\tilde{V}_{R,l-1}^{-1}}$. Then by \cite[Theorem 1]{abbasi2011improved}, for any $\delta \in (0,1)$, with probability at least $1-\delta/6$, uniformly over all $z \in \cZ$ and $l \ge 1$,
\beqa
\label{eqn:reward1}
\begin{aligned}
&\abs{\tilde{\mu}_{R,l-1}(z)-\tilde{\alpha}_{R,l-1}(z)}\\ 
&\le \frac{\sigma_R}{\sqrt{H}}\sqrt{2\left(\ln(6/\delta)+\frac{1}{2}\ln\frac{\det(\tilde{V}_{R,l-1})}{\det (HI_{d_{R,l}})}\right)}\tilde{\sigma}_{R,l-1}(z).
\end{aligned}
\eeqa
Further, from \cite[Theorem 1]{calandriello2019gaussian}, with probability at least $1-\delta/6$, uniformly over all $z \in \cZ$ and $l \in [\tau]$, we have $\frac{1}{\lambda_R} \sigma_{R,l-1}^2(z) \le \tilde{\sigma}_{R,l-1}^2(z) \le	\lambda_R \sigma_{R,l-1}^2(z)$ and $d_{R,l} \le 6\eta_R\lambda_R  \left(1+\frac{1}{H}\right) \gamma_{(l-1)H}(R)$. Then, from \cite[Equation 25]{chowdhury2019bayesian}, we have
\beq
\label{eqn:reward2}
\abs{\overline{R}_\star(z)-\tilde{\alpha}_{R,l-1}(z)} \le B_R\left(1+\frac{1}{\sqrt{1-\epsilon_R}}\right)\tilde{\sigma}_{R,l-1}(z).
\eeq
Now the result follows by combining \ref{eqn:reward1} and \ref{eqn:reward2}, and applying an union bound.
\end{proof}

Let $\sigma_{P,l}^2(z,i) \bydef k_p((z,i),(z,i))-k_{P,l}(z,i)^T(K_{P,l}+mH I_{mlH})^{-1}k_{P,l}(z,i)$ where $k_{P,l}(z,i) \bydef [k_P((z_{1,1},1),(z,i)),\ldots,k_P((z_{l,H},m),(z,i))]^T$ denotes the vector of kernel evaluations between $(z,i)$ and elements of $\bar{\cZ}_{l}$, and 
$K_{P,l}$ denotes the kernel matrix computed at $\bar{\cZ}_{l}$.

\begin{mylemma}
\label{lem:two}
For any $0<\delta,\;\epsilon_P<1$, let
$\lambda_P = \frac{1+\epsilon_P}{1-\epsilon_P}$, $\eta_P = \frac{6\lambda_P\ln(12T/\delta)}{\epsilon_P^2}$ and $b_{P,l}(z,i) = \min \lbrace \eta_P \tilde{\sigma}^2_{P,l-1}(z,i), 1 \rbrace$. Then, with probability at least $1-\delta/3$, uniformly over all $z \in \cZ$ and $l \in [\tau]$, the following holds:
\beqan
\norm{\overline{P}_\star(z)-\tilde{\mu}_{P,l-1}(z)}_2 \le \beta_{P,l}\norm{\tilde{\sigma}_{P,l-1}(z)}_2,\label{eqn:transition1}\\
\frac{1}{\lambda_P} \norm{\sigma_{P,l-1}(z)}_2^2 \le \norm{\tilde{\sigma}_{P,l-1}(z)}_2^2 \le	\lambda_P \norm{\sigma_{P,l-1}(z)}_2^2,\label{eqn:transition2}\\
d_{P,l} \le 6\eta_P\lambda_P \left(1+\frac{1}{mH}\right)\; \gamma_{m(l-1)H}(P),\label{eqn:transition3}
\eeqan
where $\sigma_{P,l-1}(z)\bydef [\sigma_{P,l-1}(z,1),\ldots,\sigma_{P,l-1}(z,m)]^T$,
and $\gamma_t(P) \equiv \gamma_t(k_P)\bydef \max_{\cA \subset \bar{\cZ}:|\cA|=t}\frac{1}{2}\ln \det(I_t+\frac{1}{mH}K_{P,\cA})$.
\end{mylemma}
\begin{proof}
The proof is similar to that of Lemma \ref{lem:one}.
\end{proof}

\begin{mytheorem}[Frequentist regret bound for Kernel-UCRL]
\label{thm:regret-bound-Nystrom}
Let the assumptions in Section \ref{sec:assumptions} hold.
For any $0<\delta,\;\epsilon_R,\;\epsilon_P<1$, let $\lambda_R =\frac{1+\epsilon_R}{1-\epsilon_R}$, $\eta_R = \frac{6\lambda_R\ln(12T/\delta)}{\epsilon_R^2}$, $b_{R,l}(z) = \min \lbrace \eta_R \tilde{\sigma}^2_{R,l-1}(z), 1 \rbrace$, $\lambda_P = \frac{1+\epsilon_P}{1-\epsilon_P}$, $\eta_P = \frac{6\lambda_P\ln(12T/\delta)}{\epsilon_P^2}$ and $b_{P,l}(z,i) = \min \lbrace \eta_P \tilde{\sigma}^2_{P,l-1}(z,i), 1 \rbrace$. Then, Kernel-UCRL with Nystrom approximation enjoys, with probability at least $1-\delta$, the regret bound 
\beqan
&\cR(T) \le (LD+2B_RH) \sqrt{2T\ln(3/\delta)} +\\ & 2C_{R,T}\sqrt{2 e \lambda_R H\gamma_{T}(R)T}+2LC_{P,T}\sqrt{2e\lambda_PmH \gamma_{mT}(P)T},
\eeqan
where $L$ is a known upper bound over $L_{M_\star}$, $D = \max_{s,s'\in \cS}\norm{s-s'}_2$ denotes the diameter of $\cS$,  $C_{P,T} = O\left(\frac{B_P}{\sqrt{1-\epsilon_P}}+\frac{\sigma_P}{\sqrt{mH}}\sqrt{\ln(1/\delta)+c_P\gamma_{mT}(P)\ln^2(mT/\delta)}\right)$, $C_{R,T} = O\left(\frac{B_R}{\sqrt{1-\epsilon_R}}+\frac{\sigma_R}{\sqrt{H}}\sqrt{\ln(1/\delta)+c_R\gamma_{T}(R)\ln^2(T/\delta)}\right)$, $c_P =\lambda_P^2/\epsilon_P^2$ and $c_R = \lambda_R^2/\epsilon_R^2$.
\end{mytheorem}

\begin{proof}
For each episode $l$, define the following events:
\beqan
\cE_{R,l}\bydef\big\lbrace  \abs{\overline{R}_\star(z)-\tilde{\mu}_{R,l-1}(z)}\le \beta_{R,l}\tilde{\sigma}_{R,l-1}(z), \forall z \big\rbrace,\\
\cE_{P,l} \bydef\big\lbrace \norm{\overline{P}_\star(z)-\tilde{\mu}_{P,l-1}(z)}_2\le \beta_{P,l}\norm{\tilde{\sigma}_{P,l-1}(z)}_2, \forall z\big\rbrace.
\eeqan 
By construction of the set of MDPs $\cM_l$, it follows that when the events $\cE_{R,l}$ and $\cE_{P,l}$ are true for all episodes $l \in [\tau]$, the unknown MDP $M_\star$ lies in $\cM_l$. Thus $V^{M_l}_{\pi_l,1}(s_{l,1}) \ge V^{M_\star}_{\pi_\star,1}(s_{l,1})$, since $M_l$ is the most optimistic MDP of $\cM_l$. This implies
\beq
\label{eqn:GP-UCRL}
V^{M_\star}_{\pi_\star,1}(s_{l,1})-V^{M_\star}_{\pi_l,1}(s_{l,1})\le V^{M_l}_{\pi_l,1}(s_{l,1})-V^{M_\star}_{\pi_l,1}(s_{l,1}).
\eeq
Now by the reproducing property of RKHS and Cauchy-Schwartz inequality
$\abs{\overline{R}_\star(z)}=\abs{\inner{\overline{R}_\star}{k_R(z,\cdot)}_{k_R}} \le \norm{\overline{R}_\star}_{k_R}k_R(z,z) \le B_R$ for all $z \in \cZ$. Thus \ref{eqn:GP-UCRL}, \cite[Lemma 7]{chowdhury2019online} and \cite[Lemma 9]{chowdhury2019online} together imply that for any $0 < \delta < 1$, with probability at least $1-\delta/3$,
\beq
\label{eqn:regret-one}
\begin{aligned}
&\cR(T)\le (LD+2B_RH) \sqrt{2T \ln(3/\delta)}+\\
&\sum_{l=1}^{\tau}\sum_{h=1}^{H}\Big( \abs{\overline{R}_{M_l}(z_{l,h}) - \overline{R}_\star(z_{l,h})}+L_{M_l} \norm{\overline{P}_{M_l}(z_{l,h}) - \overline{P}_\star(z_{l,h})}_2\Big).
\end{aligned}
\eeq
Now, by triangle inequality $\abs{\overline{R}_{M_l}(z_{l,h}) - \overline{R}_\star(z_{l,h})} \le \abs{\overline{R}_{M_l}(z_{l,h})-\tilde{\mu}_{R,l-1}(z_{l,h})}+\abs{\overline{R}_\star(z_{l,h})-\tilde{\mu}_{R,l-1}(z_{l,h})}$. Therefore, when the event $\cE_{R,l}$ is true,
\beq
\label{eqn:regret-two}
\abs{\overline{R}_{M_l}(z_{l,h}) - \overline{R}_\star(z_{l,h})} \le 2\beta_{R,l}\;\tilde{\sigma}_{R,l-1}(z_{l,h}),
\eeq
since the mean reward function $\overline{R}_{M_l}$ lies in the confidence set $\cC_{R,l}$. Similarly when the event $\cE_{P,l}$ is true,
\beq
\label{eqn:regret-three}
\norm{\overline{P}_{M_l}(z_{l,h}) - \overline{P}_\star(z_{l,h})}_2
\le 2 \beta_{P,l}\norm{\tilde{\sigma}_{P,l-1}(z_{l,h})}_2,
\eeq
since the mean transition function $\overline{P}_{M_l}$ lies in the confidence set $\cC_{P,l}$. 
Now from \cite[Lemma 10]{abbasi2011improved},
it is easy to see that $\frac{1}{2}\ln\frac{\det(\tilde{V}_{R,l-1})}{\det (HI_{d_{R,l}})}=O(d_{R,l}\ln(lH))$ and $\frac{1}{2}\ln\frac{\det(\tilde{V}_{P,l-1})}{\det (mHI_{d_{P,l}})}=O(d_{P,l}\ln(mlH))$. Since, by definition, $\gamma_t(R)$ and $\gamma_t(P)$ are non-decreasing functions in $t$, Lemmas \ref{lem:one} and \ref{lem:two} together imply that, with probability at least $1-2\delta/3$, $\beta_{P,l}=O\left(\frac{B_P}{\sqrt{1-\epsilon_P}}+\frac{\sigma_P}{\sqrt{mH}}\sqrt{\ln(1/\delta)+c_P\gamma_{mlH}(P)\ln(mlH)\ln(T/\delta)}\right)$, $\beta_{R,l}=O\left(\frac{B_R}{\sqrt{1-\epsilon_R}}+\frac{\sigma_R}{\sqrt{H}}\sqrt{\ln(1/\delta)+c_R\gamma_{lH}(R)\ln(lH)\ln(T/\delta)}\right)$. Further, it also holds that $\tilde{\sigma}_{R,l-1}(z_{l,h}) \le \sqrt{\lambda_R}\sigma_{R,l-1}(z_{l,h})$, $\norm{\tilde{\sigma}_{P,l-1}(z_{l,h})}_2 \le \sqrt{\lambda_P}\norm{\sigma_{P,l-1}(z_{l,h})}_2$, and the events $\cE_{R,l}$ and $\cE_{P,l}$ are true for all episodes $l \in [\tau]$. Now combining \ref{eqn:regret-one}, \ref{eqn:regret-two}, \ref{eqn:regret-three} and applying a union bound we have, with probability at least $1-\delta$, that
\beqn
\label{eqn:regret-four}
\begin{aligned}
&\cR(T) \le 2\sqrt{\lambda_R}\beta_{R,\tau}\sum_{l=1}^{\tau}\sum_{h=1}^{H} \sigma_{R,l-1}(z_{l,h})+\\
& 2L\sqrt{\lambda_P}\beta_{P,\tau}\sum_{l=1}^{\tau}\sum_{h=1}^{H}\norm{\sigma_{P,l-1}(z_{l,h})}_2+(LD+2B_RH) \sqrt{2T\ln(3/\delta)},
\end{aligned}
\eeqn
where we have used the fact that both $\beta_{R,l}$ and $\beta_{P,l}$ are non-decreasing with the number of episodes $l$ and that $L_{M_l} \le L$ by construction of $\cM_l$ (and since $M_l \in \cM_l$). Now the result follows from \cite[Lemma 11]{chowdhury2019online} by noting that
$\sum_{l=1}^{\tau}\sum_{h=1}^{H} \sigma_{R,l-1}(z_{l,h}) \le \sqrt{2 e H T\gamma_{T}(R)}$ and
$\sum_{l=1}^{\tau}\sum_{h=1}^{H} \norm{\sigma_{P,l-1}(z_{l,h})}_2 \le \sqrt{2em HT\gamma_{mT}(P)}$.
\end{proof} 


\subsection{Interpretation of the Bound} Theorem \ref{thm:regret-bound-Nystrom} implies that the cumulative regret of Kernel-UCRL after $\tau$ episodes is
$\tilde{O}\Big(\big(\sqrt{H\gamma_T(R)} +\gamma_T(R)\big)\sqrt{T} + L\big(\sqrt{mH\gamma_{mT}(P)} + \gamma_{mT}(P)\big)\sqrt{T}+ H\sqrt{T}\Big)$
with high probability. ($\tilde{O}$ hides logarithmic factors.)  Now we illustrate the growth of $\gamma_T(R)$ and $\gamma_{mT}(P)$ as functions of $T$ with the following concrete examples.

Let $k_R(z,z')\bydef k_1(s,s')+k_2(a,a')$, i.e., $k_R$ is an additive kernel of $k_1:\cS \times \cS \ra \Real$ and $k_2: \cA \times \cA \ra \Real$. Then, from \cite[Theorem 3]{krause2011contextual}, $\gamma_T(R) \le \gamma_T(k_1)+\gamma_T(k_2)+2 \ln T$. Now if both $\cS \subset \Real^m$, $\cA \subset \Real^n$ are compact and convex sets, and both $k_1,k_2$ are Squared Exponential (SE) kernels, then from \cite[Theorem 4]{srinivas2009gaussian}, $\gamma_T(k_1)=O\big((\ln T)^m\big)$ and $\gamma_T(k_2)=O\big((\ln T)^n\big)$. Hence, in this case
$\gamma_T(R)=\tilde{O}\big((\ln T)^{\max\lbrace m,n \rbrace}\big)$.

Further, let $k_P((z,i),(z',j))\bydef k_3(z,z') k_4(i,j)$, i.e., $k_P$ is a product kernel of $k_3:\cZ \times \cZ \ra \Real$ and $k_4: [m] \times [m] \ra \Real$. Then, from \cite[Theorem 2]{krause2011contextual},  $\gamma_{mT}(P) \le m\gamma_{mT}(k_3)+m\ln (mT)$, since all kernel matrices over any subset of $\lbrace 1,\ldots,m \rbrace$ have rank at most $m$. Now if $k_3$ is a Squared Exponential kernel on $\cZ$, then $\gamma_{mT}(k_3)=\tilde{O}\big((\ln (mT))^{m+n}\big)$. Hence, in this case $\gamma_{mT}(P)=O\big(m\big(\ln (mT)\big)^{m+n}\big)$.

In essence, $\gamma_T(R)$ and $\gamma_{mT}(P)$ grow sublinearly with $T$ for some popular kernels, e.g. Squared Exponential, polynomial and Mat\'{e}rn. Now, since the cumulative regret of Kernel-UCRL scales linearly with $\gamma_T(R)$ and $\gamma_{mT}(P)$, it, in turn, grows sublinearly with $T$ for these kernels.

\section{ANALYSIS OF PSRL UNDER NYSTR\"{O}M APPROXIMATION}

Optimizing for an optimistic policy is not computationally tractable in general, even though planning for the optimal policy is possible for a given MDP. A popular approach to overcome this difficulty is to sample a random MDP at every episode and solve for its optimal policy, called posterior sampling \cite{osband2016posterior}.

PSRL (Algorithm \ref{algo:Kernel-PSRL}), in its most general form, starts with a prior distribution $\Phi$ over MDPs. At the beginning of episode $l$, using the history of observations $\cH_{l-1}$, it updates the posterior $\Phi_l$ and samples an MDP $M_l$ from $\Phi_l$.
(Sampling can be done using MCMC methods even if $\Phi_{l }$ doesn't admit any closed form.) It then selects an optimal policy $\pi_l$ of the sampled MDP $M_l$, in the sense that $V^{M_l}_{\pi_l,h}(s)=\max_{\pi}V^{M_l}_{\pi,h}(s)$ for all $s \in \cS$ and for all $h=1,2,\ldots,H$,
and executes $\pi_l$ for the entire episode.


\begin{algorithm}
\renewcommand\thealgorithm{2}
\caption{PSRL}\label{algo:Kernel-PSRL}
\begin{algorithmic}
\STATE \textbf{Input:} Prior $\Phi$. 
\STATE Set $\Phi_1=\Phi$.
\FOR{ episode $l = 1, 2, 3, \ldots$}
\STATE Sample $M_l \sim \Phi_l$.
\STATE Choose policy 
$\pi_l$ such that $V^{M_l}_{\pi_l,h}(s)=\max_{\pi}V^{M_l}_{\pi,h}(s)\; \forall s \in \cS,\forall h=1,2,\ldots,H$.
\STATE \FOR{period $h=1,2,3, \ldots, H$} 
\STATE Choose action $a_{l,h} = \pi_l(s_{l,h},h)$.
\STATE Observe reward $r_{l,h} = \overline{R}_\star(z_{l,h}) + \epsilon_{R,l,h}$.
\STATE Observe next state $s_{l,h+1}=\overline{P}_\star(z_{l,h}) + \epsilon_{P,l,h}$.
\ENDFOR
\STATE Update $\Phi_l$ to $\Phi_{l+1}$, using $\left \lbrace s_{l,h}, a_{l,h}, s_{l,h+1} \right\rbrace_{1 \le h \le H}$.
\ENDFOR
\end{algorithmic}
\addtocounter{algorithm}{-2}
\end{algorithm}

\cite{osband2016posterior} show that if we have a frequentist regret bound for UCRL in hand, then we can obtain a similar bound (upto a constant factor) on the \emph{Bayes regret}, defined as the \textit{expected regret under the prior distribution $\Phi$}, of PSRL. We use this idea to obtain a sublinear bound on the Bayes regret of PSRL under kernel approximation.
\begin{mytheorem}[Bayes regret of PSRL under Nystr\"{o}m approximation]
\label{thm:regret-bound-PSRL}
Let the assumptions in Section \ref{sec:assumptions} hold and $\Phi$ be a (known) prior distribution over MDPs $M_\star$. Then, the Bayes regret of PSRL satisfies
\beqan
&\expect{\cR(T)} \le 3B_R+2 \hat{C}_{R,T} \sqrt{2 e  H\gamma_{T}(R)T} \\&+ 3\expect{L_{M_\star}}\hat{C}_{P,T}\sqrt{2em H\gamma_{mT}(P)T},
\eeqan
where $L_{M_\star}$ is the global Lipschitz constant for the future value function of $M_\star$, $\hat{C}_{R,T} \bydef O\left(\frac{B_R}{\sqrt{1-\epsilon_R}}+\frac{\sigma_R}{\sqrt{H}}\sqrt{\ln(T)+c_R\gamma_{T}(R)\ln^2(T)}\right)$ and $\hat{C}_{P,T}\bydef O\left(\frac{B_P}{\sqrt{1-\epsilon_P}}+\frac{\sigma_P}{\sqrt{mH}}\sqrt{\ln(T)+c_P\gamma_{mT}(P)\ln^2(mT)}\right)$.
\end{mytheorem}
\begin{proof}
The proof is similar to that of \cite[Theorem 2]{chowdhury2019online}.
\end{proof}

\section{CONCLUSIONS}
Any MDP $M$ whose mean reward function satisfies $\overline{R}_M(z) = \theta_R^T\tilde{\phi}_{R,l}(z)$ for some $\theta_R \in \Real^{d_{R,l}}$ such that $\norm{\theta_R-\tilde{\theta}_{R,l-1}}_{\tilde{V}_{R,l-1}} \le \sqrt{H}\beta_{R,l}$, and mean transition function satisfies $\overline{P}_M(z,i) = \theta_P^T\tilde{\phi}_{P,l}(z,i)$, $i=1,\ldots,m$, for some $\theta_P \in \Real^{d_{P,l}}$ such that $\norm{\theta_P-\tilde{\theta}_{P,l-1}}_{\tilde{V}_{P,l-1}} \le \sqrt{mH}\beta_{P,l}$, lies in the set $\cM_l$. However, there might be other MDPs in $\cM_l$ which do not posses this linear structure. Therefore, optimal planning may be computationally intractable even for a single MDP. So it is common in the literature to assume access to an approximate MDP planner $\Gamma (M,\eps)$ which returns an $\epsilon$-optimal policy for $M$. Given such a planner $\Gamma$, if it is possible to obtain (through extended value iteration \cite{jaksch2010near} or otherwise) an efficient planner $\tilde{\Gamma} (\cM,\eps)$ which returns an $\epsilon$-optimal policy for the most optimistic MDP from a  family $\cM$, then we modify PSRL and Kernel-UCRL to choose $\pi_l=\Gamma (M_l,\sqrt{H/l})$ and $\pi_l=\tilde{\Gamma} (\cM_l,\sqrt{H/l})$ respectively at every episode $l$. It follows that this adds only an $O(\sqrt{T})$ factor in the respective regret bounds. The design of such approximate planners for continuous state and action spaces remains a subject of active research, whereas our focus in this work is on the statistical efficiency of the online learning problem.





\section*{APPENDIX}
In this section we will assume that $\cS \subset \Real$, i.e.,  $m=1$ and $\cA \subset \Real^n$ for some $n=O(1)$. Also we will assume that $k_R$ and $k_P$ are Squared Exponential (SE) kernels defined over $[0,1]^{n+1}$ with length scale parameters $l_R$ and $l_P$, respectively.
\subsection{Kernel-UCRL under QFF Approximation}
First, we choose an $m_R \in \mathbb{N}$ such that $1/l_R^2 \le m_R \le C_1/l_R^2$ and $\log_{4/e}(T^6) \le m_R \le C_2 \log_{4/e}(T^6)$ for some appropriate constants $C_1$ and $C_2$. Then we set $d_R=2(m_R)^{n+1}$ and construct $d_R$ dimensional feature map $\tilde{\phi}_R(z)$ using \ref{eqn:qff-embedding}. Similarly, we choose an $m_P \in \mathbb{N}$ such that $1/l_P^2 \le m_P \le C_3/l_P^2$ and $\log_{4/e}(T^6) \le m_P \le C_4 \log_{4/e}(T^6)$ for some appropriate constants $C_3$ and $C_4$ and construct the feature map $\tilde{\phi}_P(z)$ of dimension $d_P=2(m_P)^{n+1}$. Now we construct confidence sets $\cC_{R,l}$ and $\cC_{P,l}$ as follows. First, we compute $\tilde{\theta}_{R,l-1}=\tilde{V}_{R,l-1}^{-1}\tilde{\Phi}_{R,l-1}^T R_{l-1}$, where $R_{l-1} = [r_{1,1},\ldots,r_{l-1,H}]^T$, $\tilde{\Phi}_{R,l-1}=[\tilde{\phi}_{R}(z_{1,1}),\ldots,\tilde{\phi}_{R}(z_{l-1,H})]^T$ and $\tilde{V}_{R,l-1} = \tilde{\Phi}_{R,l-1}^T\tilde{\Phi}_{R,l-1}+H I_{d_{R}}$. Then we fix any $0<\delta<1$ and define $\cC_{R,l}$ to be the set of all functions $f:\cZ \ra \Real$ such that
\beqn
\label{eqn:confidence-set-reward-ff} 
\abs{f(z)-\tilde{\mu}_{R,l-1}(z)} \le \beta_{R,l}\tilde{\sigma}_{R,l-1}(z)+O(B_R/T), \forall z \in \cZ,
\eeqn
where $\tilde{\mu}_{R,l-1}(z) = \tilde{\phi}_{R}(z)^T\tilde{\theta}_{R,l-1}$, $\tilde{\sigma}_{R,l-1}^2(z)= H\tilde{\phi}_{R}(z)^T \tilde{V}_{R,l-1}^{-1}\tilde{\phi}_{R}(z)$
and $\beta_{R,l}=B_R+ \frac{\sigma_R}{\sqrt{H}}\sqrt{2\left(\ln(3/\delta)+\frac{1}{2}\ln\frac{\det(\tilde{V}_{R,l-1})}{\det (HI_{d_{R}})}\right)}$.

Similarly, we compute $\tilde{\theta}_{P,l-1}=\tilde{V}_{P,l-1}^{-1}\tilde{\Phi}_{P,l-1}^T S_{l-1}$, where $S_{l-1} = [s_{1,2},\ldots,s_{l-1,H+1}]^T$, $\tilde{\Phi}_{P,l-1}=[\tilde{\phi}_{P}(z_{1,1}),\ldots,\tilde{\phi}_{P}(z_{l-1,H})]^T$ and $\tilde{V}_{P,l-1} = \tilde{\Phi}_{P,l-1}^T\tilde{\Phi}_{P,l-1}+H I_{d_{P}}$. Now, we define $\cC_{P,l}$ to be the set of all functions $f:\cZ \ra \Real^m$ such that 
\beqn
\label{eqn:confidence-set-state-ff} 
\abs{f(z)-\tilde{\mu}_{P,l-1}(z)} \le \beta_{P,l}\tilde{\sigma}_{P,l-1}(z)+O(B_P/T), \forall z \in \cZ,
\eeqn
where $\tilde{\mu}_{P,l-1}(z)= \tilde{\phi}_{P}(z)^T\tilde{\theta}_{P,l-1}$ and $\tilde{\sigma}_{P,l-1}^2(z)= H\tilde{\phi}_{P}(z)^T \tilde{V}_{P,l-1}^{-1}\tilde{\phi}_{P}(z)$ and
$\beta_{P,l}=B_P+ \frac{\sigma_P}{\sqrt{H}}\sqrt{2\left(\ln(3/\delta)+\frac{1}{2}\ln\frac{\det(\tilde{V}_{P,l-1})}{\det (HI_{d_{P}})}\right)}$.

Next, following the same approach as before, we build the set $\cM_l$ of all plausible MDPs, choose an optimistic policy $\pi_l$ for $\cM_l$ 
and execute $\pi_l$ for the entire episode.

\subsection{Regret Bound under QFF Approximation}

\begin{mylemma}
\label{lem:three}
Let $m_R \ge \max\lbrace 1/l_R^2,\log_{4/e}(T^6)\rbrace$ and $\delta \in (0,1)$. Then, with probability at least $1-\delta/3$, uniformly over all $z \in \cZ$ and $1 \le l \le \tau$, 
\beqn
\abs{\overline{R}_\star(z)-\tilde{\mu}_{R,l-1}(z)} \le \beta_{R,l}\tilde{\sigma}_{R,l-1}(z) + O(B_R/T).
\eeqn
\end{mylemma}


\begin{proof}
Note that under QFF approximation $\tilde{\sigma}_{R,l-1}^2(z)= H\tilde{\phi}_{R,l}(z)^T \tilde{V}_{R,l-1}^{-1}\tilde{\phi}_{R,l}(z)$, where  $\tilde{\phi}_{R,l}=\tilde{\phi}_R$ and $d_{R,l} =d_R$ for every episode $l$. Now define $\alpha_{R,l}(z)=k_{R,l}(z)^T(K_{R,l}+H I_{lH})^{-1}\overline{R}_{\star,l-1}$.
Then from \cite[Equation 7]{chowdhury2019bayesian}, we have $\abs{\overline{R}_\star(z)-\alpha_{R,l-1}(z)} \le B_R\sigma_{R,l-1}(z)$. Let 
$\epsilon_{d_R} \bydef \sup_{z,z' \in \cZ} \abs{k_R(z,z')-\tilde{\phi}_R(z)^T\tilde{\phi}_R(z')} < 1$. Then from \cite[Lemma 15]{chowdhury2019bayesian},
\beqan
\quad\abs{\alpha_{R,l-1}(z)-\tilde{\alpha}_{R,l-1}(z)} &=& O(B_RH\epsilon_{d_R} l^2),\\ \abs{\sigma_{R,l-1}(z) - \tilde{\sigma}_{R,l-1}(z)} &=&   O (\epsilon_{d_R}^{1/2}l).
\eeqan
Therefore, by the triangle inequality,
\beqa
\abs{\overline{R}_\star(z)-\tilde{\alpha}_{R,l-1}(z)} \le  B_R\sigma_{R,l-1}(z) + O(B_RH\epsilon_{d_R} l^2)\nonumber\\
= B_R\tilde{\sigma}_{R,l-1}(z)+O(B_RH\epsilon_{d_R}^{1/2}l^2).
\label{eqn:reward-ff}
\eeqa
From \cite[Theorem 1]{mutny2018efficient} $\epsilon_{d_R} \le (n+1)2^{n}\frac{1}{\sqrt{2}m_R^{m_R}}\left(\frac{e}{4l_R^2}\right)^{m_R}=O\left(\left(\frac{e}{4m_Rl_R^2}\right)^{m_R}\right)$, since $n=O(1)$. If $m_R \ge 1/l_R^2$, then $\epsilon_{d_R}=O\left((e/4)^{m_R}\right)$. Further if $m_R \ge \log_{4/e}(T^6)$, then $\epsilon_{d_R}=O(1/T^6)$ and thus, in turn, $H\epsilon_{d_R}^{1/2}l^2=O(1/H^2\tau)=O(1/T)$ for each $l \le \tau$. Now the result follows by combining \ref{eqn:reward1} and \ref{eqn:reward-ff} using the triangle inequality.
\end{proof}

\begin{mylemma}
\label{lem:four}
Let $m_P \ge \max\lbrace 1/l_P^2,\log_{4/e}(T^6)\rbrace$ and $\delta \in (0,1)$. Then, with probability at least $1-\delta/3$, uniformly over all $z \in \cZ$ and $1 \le l \le \tau$, 
\beqn
\abs{\overline{P}_\star(z)-\tilde{\mu}_{P,l-1}(z)} \le \beta_{P,l}\tilde{\sigma}_{P,l-1}(z) + O(B_P/T).
\eeqn
\end{mylemma}

\begin{proof}
The proof is similar to that of Lemma \ref{lem:three}.
\end{proof}

\begin{mytheorem}[Regret bound for Kernel-UCRL]
\label{thm:regret-bound-FF}
Let the assumptions in Section \ref{sec:assumptions} hold. Further, let $m=1$, $n=O(1)$ and $k_R,\;k_P$ are SE kernels on $[0,1]^{n+1}$ with length scale parameters $l_R,\;l_P$, respectively. Then, for any $\delta \in (0,1)$, Kernel-UCRL with QFF approximation enjoys, with probability at least $1-\delta$, the regret bound 
\beqan
&\cR(T) = O \Big((B_R+B_P)+(LD+B_RH) \sqrt{T\ln(1/\delta)} +\\ & C_{R,T}\sqrt{HT(\ln T)^{n+2}}+LC_{P,T}\sqrt{HT(\ln T)^{n+2}}\Big),
\eeqan
where $C_{P,T} = O\left(B_P+\frac{\sigma_P}{\sqrt{H}}\sqrt{\ln(1/\delta)+(\ln  T)^{n+2}}\right)$ and $C_{R,T} = O\left(B_R+\frac{\sigma_R}{\sqrt{H}}\sqrt{\ln(1/\delta)+(\ln T)^{n+2}}\right)$.
\end{mytheorem}
\begin{proof}
For each episode $l$, define the following events:
\beqan
\cE_{R,l}\bydef\big\lbrace  \abs{\overline{R}_\star(z)-\tilde{\mu}_{R,l-1}(z)}\le \beta_{R,l}\tilde{\sigma}_{R,l-1}(z)+O(B_R/T), \forall z \big\rbrace,\\
\cE_{P,l} \bydef\big\lbrace \abs{\overline{P}_\star(z)-\tilde{\mu}_{P,l-1}(z)}\le \beta_{P,l}\tilde{\sigma}_{P,l-1}(z)+O(B_P/T), \forall z\big\rbrace.
\eeqan 
When the events $\cE_{R,l}$ and $\cE_{P,l}$ are true for all episodes $l \in [\tau]$, then using a similar approach as in the proof of Theorem \ref{thm:regret-bound-Nystrom} we can show that for any $0 < \delta < 1$, with probability at least $1-\delta/3$,
\beq
\label{eqn:regret-one-ff}
\begin{aligned}
&\cR(T)\le (LD+2B_RH) \sqrt{2T \ln(3/\delta)}+\\
&\sum_{l=1}^{\tau}\sum_{h=1}^{H}\Big( \abs{\overline{R}_{M_l}(z_{l,h}) - \overline{R}_\star(z_{l,h})}+L \abs{\overline{P}_{M_l}(z_{l,h}) - \overline{P}_\star(z_{l,h})}\Big).
\end{aligned}
\eeq
Also for every episode $l$ the following holds:
\beqan
\label{eqn:regret-two-ff}
\abs{\overline{R}_{M_l}(z_{l,h}) - \overline{R}_\star(z_{l,h})} &\le& 2\beta_{R,l}\;\tilde{\sigma}_{R,l-1}(z_{l,h})+O(B_R/T),\\
\abs{\overline{P}_{M_l}(z_{l,h}) - \overline{P}_\star(z_{l,h})}
&\le& 2 \beta_{P,l}\tilde{\sigma}_{P,l-1}(z_{l,h}) + O(B_P/T).
\eeqan
By our choice of $m_R$ and $m_P$, Lemmas \ref{lem:three} and \ref{lem:four} together imply that the events $\cE_{R,l}$ and $\cE_{P,l}$ are true for all episodes $l \in [\tau]$. Further, since $\frac{1}{2}\ln\frac{\det(\tilde{V}_{R,l-1})}{\det (HI_{d_{R}})}=O(d_{R}\ln(lH))$ and $\frac{1}{2}\ln\frac{\det(\tilde{V}_{P,l-1})}{\det (HI_{d_{P}})}=O(d_{P}\ln(lH))$, $\beta_{R,l}$ and $\beta_{P,l}$ are non-decreasing functions in $l$.
 Now, combining \ref{eqn:regret-one-ff} and \ref{eqn:regret-two-ff}, and applying a union bound we have, with probability at least $1-\delta$, that
\beqn
\begin{aligned}
&\cR(T) \le O(B_R+B_P)+(LD+2B_RH) \sqrt{2T\ln(3/\delta)}\\ & 2\beta_{R,\tau}\sum_{l=1}^{\tau}\sum_{h=1}^{H} \tilde{\sigma}_{R,l-1}(z_{l,h})+
 2L\beta_{P,\tau}\sum_{l=1}^{\tau}\sum_{h=1}^{H}\tilde{\sigma}_{P,l-1}(z_{l,h}).
\end{aligned}
\eeqn
From \cite[Lemma 11]{chowdhury2019online}, $\tilde{\sigma}_{R,l-1}(z_{l,h}) = O\left(\sqrt{HTd_R\ln T}\right)$ and $\tilde{\sigma}_{P,l-1}(z_{l,h}) = O\left(\sqrt{HTd_P\ln T}\right)$, since $\gamma_t(k)=O(d\ln t)$ for any linear kernel $k$ defined over $\Real^d$. Now the result follows by noting that $d_R =(m_R)^{n+1}= O\left((\ln T)^{n+1}\right)$ and $d_P =(m_P)^{n+1}= O\left((\ln T)^{n+1}\right)$ for $n=O(1)$.
\end{proof}

\addtolength{\textheight}{-6cm}


\section*{ACKNOWLEDGMENT}
Sayak Ray Chowdhury is supported by a Google PhD Fellowship. Aditya Gopalan is grateful
for support from the DST INSPIRE faculty grant IFA13-
ENG-69.


\bibliographystyle{IEEEtran}
\bibliography{IEEEabrv,paper}

\end{document}